\documentclass{article}
\usepackage[a4paper, margin=2.5cm]{geometry}

\usepackage{times}
\usepackage{url}
\usepackage[utf8]{inputenc}
\usepackage[small]{caption}
\usepackage{graphicx}
\usepackage{amsmath,amssymb}
\usepackage{amsthm}
\usepackage{booktabs}
\usepackage{algorithm}
\usepackage{algorithmic}
\usepackage[switch]{lineno}

\usepackage{array}
\usepackage{amsfonts}       
\usepackage{amsthm}
\usepackage{amsmath}
\usepackage{algorithm}
\usepackage{algorithmic}
\usepackage{theoremref}
\usepackage[table]{xcolor}
\usepackage{titlesec}
\usepackage{multirow}
\usepackage{diagbox}
\usepackage{threeparttable}
\usepackage{verbatim}
\usepackage{booktabs}
\usepackage{subcaption}
\usepackage{setspace}

\newtheorem{definition}{Definition}

\newtheorem{observation}{Observation}

\newtheorem{assumption}{Assumption}

\newtheorem{lemma}{Lemma}

\newcommand{\pa}{\operatorname{PA}}

\newcommand{\pc}{\operatorname{PC}}

\newcommand{\cde}{\operatorname{CDE}}

\newcommand{\cate}{\operatorname{CATE}}
\newcommand{\ate}{\operatorname{ATE}}

\renewcommand{\emph}{\textit}
\renewcommand{\em}{\textit}

\makeatletter
\newcommand*{\indep}{%
	\mathbin{%
		\mathpalette{\@indep}{}%
	}%
}

\newcommand*{\nindep}{%
	\mathbin{
		\mathpalette{\@indep}{/}%
	}%
}

\newcommand*{\@indep}[2]{%
	\sbox0{$#1\perp\m@th$}
	\sbox2{$#1=$}
	\sbox4{$#1\vcenter{}$}
	\rlap{\copy0}
	\dimen@=\dimexpr\ht2-\ht4-.2pt\relax
	\kern\dimen@
	\ifx\\#2\\%
	\else
	\hbox to \wd2{\hss$#1#2\m@th$\hss}%
	\kern-\wd2 %
	\fi
	\kern\dimen@
	\copy0 
}
\makeatother

\newtheorem{example}{Example}
\newtheorem{theorem}{Theorem}

\title{Linking a predictive model to causal effect estimation}

\author{
    Anonymous
}

\author{
Jiuyong Li
\and
Lin Liu \and
Ziqi Xu \and
Ha Xuan Tran \and
Thuc Duy Le \and
Jixue Liu \\ \\
University of South Australia. Adelaide, Australia
}

\begin{document}

\date{}

\maketitle

\begin{abstract}

A predictive model makes outcome predictions based on some given features, i.e., it estimates the conditional probability of the outcome given a feature vector. In general, a predictive model cannot estimate the causal effect of a feature on the outcome, i.e., how the outcome will change if the feature is changed while keeping the values of other features unchanged. This is because causal effect estimation requires interventional probabilities. However, many real world problems such as personalised decision making, recommendation, and fairness computing, need to know the causal effect of any  feature on the outcome for a given instance. This is different from the traditional causal effect estimation  problem with a fixed treatment variable. This paper first tackles the challenge of estimating the causal effect of any feature (as the treatment) on the outcome w.r.t. a given instance. The theoretical results naturally link a predictive model to causal effect estimations and imply that a predictive model is causally interpretable when the conditions identified in the paper are satisfied. The paper also reveals the robust property of a causally interpretable model.     
We use experiments to demonstrate that various types of predictive models, when satisfying the conditions identified in this paper, can estimate the causal effects of features as accurately as state-of-the-art causal effect estimation methods. We also show the potential of such causally interpretable predictive models for robust predictions and personalised decision making.

	
\end{abstract}

\section{Introduction}

Most supervised learning methods are designed to estimate conditional probability $P(Y \mid \mathbf{X = x})$ where $\mathbf{X}$ is known as the set of features (variables) and $Y$ the outcome (or the target variable). A large number of such predictive models have been developed for outcome predictions and they can model conditional probabilities accurately in various data sets. 

However, many decision making problems require answers to what-if questions regarding any feature and the outcome. If $X_i \in \mathbf{X}$ is changed from 0 to 1, how will $Y$ be changed?  For example, will a job applicant get the job if the applicant has a college degree? Will a customer buy the product if the customer receives the coupon?  Answering such what-if questions needs the estimation of the causal effect of the corresponding feature (e.g. college degree or coupon) on the outcome (e.g. getting the job or buying the product). The causal effect of a feature (as a treatment) on the outcome means the change of the outcome due to the change of the feature~\cite{ImbensRubin2015_Book,Pearl2009_Book}.



Causal effect estimation is a main topic in the causal inference area, and most causal effect estimation methods assume a fixed treatment variable and a given set of covariates~\cite{ImbensRubin2015_Book,hernan2010causal,CHM-Evaluation}. Several Conditional Average Treatment Effect (CATE) estimation methods have been proposed for personalised decision making ~\cite{AtheyImbens2016_PNAS,wager2018estimation,Shalit2016,CEVAE2017,zhang2021treatment}, and they all deal with a fixed treatment variable and a given set of covariates.    

In many problem settings, there is not a fixed treatment variable, and therefore, these causal effect estimation methods cannot be applied to such a problem. For example, with a data set containing people's lifestyles, diets and etc. (features), and heart attacks in their 60s (the outcome). A predictive model can be trained on the data set to predict the risk of a heart attack in their 60s based on their feature values. Furthermore, one may wish to know what lifestyle change or diet change she/he should undertake now to reduce the risk of a heart attack in future. Each of lifestyle and diet features can be a treatment and a covariate when it is not a treatment. 

Identifying the best possible treatment for an individual is a key part for the solution in personalised decision making. During the search for the best possible treatment for an individual, features  are alternated as treatments to have their respective causal effects evaluated. When a feature is considered as the treatment, an appropriate covariate set for the treatment and the outcome needs to be determined. This makes an existing  causal effect estimation method not applicable since it assumes a fixed treatment variable and a given covariate set for the treatment and the outcome~\cite{ImbensRubin2015_Book,hernan2010causal}. 
Existing treatment effect estimation methods need to be extended to a predictive machine learning setting where each feature can be a treatment and there is not a designation of a fixed treatment and a covariate set.  

In this paper, firstly we deal with the challenge to causal effect estimation when alternating features as treatments. It is not a straightforward extension of existing causal effect estimation methods when alternating treatment variables. We show the challenges in the following example. Consider a traditional setting where a data set contains the treatment variable $T$, a covariate variable $X$, and the outcome $Y$. This implies that the causal graph underlying the data set is $T \leftarrow X \to Y$ and $T \to Y$. In the problem setting of this paper, both $(T, X)$ are features and can be treatments alternatively. To estimate the causal effect of $T$ on $Y$ conditioned on $X=x$, the conditional average causal effect (CATE) of $T$ on $Y$ conditioned on $X=x$  is used. However, when $X$ is alternated to be the treatment, $T$ is not a covariate of the variable pair $(X, Y)$ but a mediator between $X$ and $Y$. In this case, a CATE estimator is not applicable for estimating the causal effect of $X$ on $Y$ conditioned on $T$ since, when changing $X$ as the treatment, $T$ will be changed too. Instead, the Controlled Direct Effect (CDE) of $X$ on $Y$ when $T$ is controlled to $T=t$ is appropriate. 

We face two challenges as shown in the above example. Firstly, when alternating features as treatments, the causal effects for different features may not be of the same type. Secondly, the choice of the right type of causal effect for a treatment needs the underlying causal graph. So, in general, we cannot estimate the causal effect of individual features on the outcome by alternating them as treatments if we do not know the underlying causal graph. In practice, the causal graph is rarely known. In this paper, we reduce the requirement of knowing the causal graph to the requirement of knowing that the data set contains all direct causes of the outcome and no variables affected by the outcome. Hidden variables which are not the direct causes are allowed. This makes our proposed approach more practical than methods assuming knowing the causal graph for estimating the causal effects of any feature on the outcome.  

Secondly, the theoretical results of the above solution naturally link a predictive model to causal effect estimation. i.e., when the conditions are satisfied, causal effects can be directly estimated from a predictive model. The results imply that we can have a causally interpretable predictive model where causal effects of individual features can be derived from the model. The causal interpretation does not depend on the transparency of the predictive model, but depends on what inputs to the model are. The predictive model itself can be black-box, but the effect of each feature on the outcome can be derived from the model and hence the model is causally interpretable. 

The connection between predictive models and  causal interpretation contributes to model explanation. It is desirable to interpret a predictive model causally, but it is generally impossible. As a predictive model, the coefficients in linear regression models are often interpreted as causal effects of the corresponding features on the outcome. However, such an interpretation is valid only when the underlying causal structure is known and the regression follows the causal structure properly. This is because the coefficients may represent different types of causal effects, or even may be biased estimations~\cite{westreich2013table}.  Furthermore, even if the causal interpretation is valid, the coefficients in linear regression models indicate the Average Treatment Effects (ATE)  at population level, and are not suitable for personalised decision making discussed in this paper. 

Thirdly, we further study the property of the above causally interpretable model, and show that it is robust when the environment is changed from which the model is trained.     







In summary, this work makes the following contributions. 




\begin{enumerate}
	
	
\item For causal effect estimation of any feature on the outcome in a given circumstance, normally the causal graph underlying the data set is needed. This paper reduces the requirement of knowing the complete causal graph to that of knowing that the data set contains all direct causes of the outcome and no variables affected by the outcome. Hidden variables are allowed as long as they are not direct causes of the outcome. The relaxed assumptions lead to more practical methods for causal effect estimation of any feature on the outcome. 

	
	\item 
The work links a predictive model to causal effect estimation.  When the conditions identified in the paper are satisfied, a predictive model can be used for causal effect estimation and hence causally interpretable. We use experiments to show that various types of predictive models estimate causal effects as accurately as state-of-the-art causal effect estimation methods.   

	
\item We analyse and demonstrate that a causally interpretable model is more robust than other predictive models using all features in a new environment that is different from the one in which the model is trained. 

\end{enumerate}



	\section{Graphical causal models}
	\label{sec:graphicalvausalModels}
	In this section, we present the necessary background of causal inference. 
	We use upper case letters to represent variables and bold-faced upper case letters to denote sets of variables. The values of variables are represented using lower case letters.
	
	Let $\mathcal{G}=(\mathbf{V},\mathbf{E})$ be a graph, where $\mathbf{V}=\{V_{1},\dots, V_{p}\}$ is the set of nodes and $\mathbf{E}$ is the set of edges between the nodes. A path $\pi$ is a sequence of distinct nodes such that every pair of successive nodes are adjacent in $\mathcal{G}$. A path $\pi$  is directed if all arrows of edges along the path point towards the same direction.  A path between $(V_i,  V_j)$ is a back door path with respect to $V_i$ if it has an arrow into $V_i$. Given a path $\pi$, $V_{k}$ is a collider node on $\pi$ if there are two edges incident such that  $V_{i}\rightarrow V_{k} \leftarrow V_{j}$. In $\mathcal{G}$, if there exists $V_i\rightarrow V_j$, $V_i$ is a parent of $V_j$ and we use $\pa(V_j)$ to denote the set of all parents of $V_j$. If there exists a directed path from $V_i$ to $V_j$, then $V_i$ is an ancestor of $V_j$ and $V_j$ is a descendant of $V_i$.
	
	A DAG (Directed Acyclic Graph) is a directed graph without directed cycles.  If we interpret a node's parent as its direct cause, and a directed path as a causal path where the start node is an indirect cause of the end node, the DAG is known as a causal DAG. 
	

	%
	%
	%
	
	With the following two assumptions, a DAG links to a distribution. 
	
	\begin{assumption} [Markov condition~\cite{Pearl2009_Book}]
		\label{asm_Markovcondition}
		Given a DAG $\mathcal{G}=(\mathbf{V}, \mathbf{E})$ and $P(\mathbf{V})$, the joint probability distribution of $\mathbf{V}$, $\mathcal{G}$ satisfies the Markov condition if for $\forall V_i \in \mathbf{V}$, $V_i$ is independent of all non-descendants of $V_i$, given the parents of $V_i$.
	\end{assumption}
	
	
	\begin{assumption}[Faithfulness~\cite{Spirtes2000_Book}]
		A DAG $\mathcal{G}=(\mathbf{V}, \mathbf{E})$ is faithful to $P(\mathbf{V})$ iff every conditional independence presenting in $P(\mathbf{V})$  is entailed by $\mathcal{G}$, which fulfils the Markov condition. 
		\label{asm_Faithfulness}	
	\end{assumption}
	
	With the Markov condition and faithfulness assumptions, we can read the (in)dependencies between variables in $P(\mathbf{V})$ from a DAG using $d$-separation~\cite{Pearl2009_Book}. 
	
	\begin{definition} [$d$-Separation]
		A path $\pi$ in a DAG is $d$-separated by a set of nodes $\mathbf{S}$ if and only if: (1) $\mathbf{S}$ contains the middle node, $V_k$, of a chain $V_i \to V_k \to V_j$, $V_i \leftarrow V_k \leftarrow V_j$, or  $V_i \leftarrow V_k \to V_j$ in path $\pi$; (2) when path $\pi$ contains a collider $V_i \to V_k \leftarrow V_j$, node $V_k$ or any descendant of $V_k$ is not in $\mathbf{S}$. 
	\end{definition}
	
	
	
	%

The set $\mathbf{S}$ can be an empty set. A path is $d$-separated by an empty set if the path includes a collider. If a path is $d$-separated by $\mathbf{S}$, then $V_p \indep V_q \mid \mathbf{S}$ where $V_p$ and $V_q$ are start and end nodes of the path. If the path is not $d$-separated by any set including the empty set, then the path is $d$-connected and $V_p$ and $V_q$ are associated.	

	The causal (treatment) effect indicates the strength of a causal relationship. To define the causal effect, we introduce the concept of intervention, which forces a variable to take a value, often denoted by a \emph{do} operator ~\cite{Pearl2009_Book}. A \emph{do} operation mimics an intervention in a real world experiment. For example, $do(X=1)$ means $X$ is intervened to take value 1. $P(y \mid do (X=1))$ is an interventional probability. 
	
	The average treatment effect quantifies the treatment effect of a pair of variables in the population. 
	
	\begin{definition}[Average Treatment Effect (ATE)]
		\label{def:ATE}
		The average treatment effect of $X$ on $Y$  is defined as $\ate(X, Y) = P(Y=1 \mid do(X = 1)) - P(Y=1 \mid do(X = 0))$.    
	\end{definition}
	
The average treatment effect is the mean treatment effect in the population. However, the treatment effects of some subgroups may deviate greatly from the average treatment effect. The conditional treatment effect is used to capture the treatment effect in a subgroup. 

\begin{definition}[Conditional Average Treatment Effect (CATE)] 
	\label{def:CATE}
	Let $\mathbf{W}$ be a set of variables. Conditional average treatment effect of $X$ on $Y$ given $\mathbf{W=w}$ is defined as $\cate(X, Y; \mathbf{W=w}) = P(Y=1 \mid do(X = 1), \mathbf{W=w}) - P(Y=1 \mid do(X = 0), \mathbf{W=w})$.  		 
\end{definition}

To be an eligible condition, $\mathbf{W}$ usually does not include any descendant nodes of $X$ or $Y$.  
	
When there are multiple directed paths between $X$ and $Y$ including the path $X \to Y$, the variables in the directed paths apart from $X$ and $Y$ are mediators. Controlled direct effect measures the direct effect via edge $X \to Y$ when the mediators $\mathbf{Z}$ are kept constant to $\mathbf{z}$.

\begin{definition}[Controlled Direct Effect (CDE)~\cite{pearl2012causal}]
	\label{def:CDE}
	The controlled direct effect of $X$ on $Y$ when setting mediators between $(X, Y)$, denoted as $\mathbf{Z}$, as a constant value $\mathbf{z}$, i.e., $\mathbf{Z = z}$ is defined as: $\cde(X, Y; \mathbf{Z=z}) = P(Y \mid do (X=1), do (\mathbf{Z=z})) - P(Y \mid do (X=0), do (\mathbf{Z=z}))$.
\end{definition}

	To infer interventional probabilities (\emph{do} operators) by reducing them to normal conditional probabilities, a causal DAG is necessary.  The back door criterion, front door criterion or the $do$ calculus~\cite{Pearl2009_Book} can be used for the reduction.

\section{Assumptions, identifiability and algorithm}

\subsection{CDE, CATE and identifiability when alternating features as treatments}
\label{sec_DCE}	
	
Let us consider a typical machine learning problem setting where $\mathbf{X} = \{X_1, X_2, \dots, X_m\}$ is a set of features and $Y$ is the outcome. For simplicity of presentation, we assume $\mathbf{X}$ and $Y$ are binary. However, all conclusions in this section apply to both binary and continuous variables since our conclusions indicate that interventional probabilities and observational probabilities are exchangeable when the assumptions are satisfied regardless of variable types. For example, when $Y$ is continuous, $E(y \mid \mathbf{X=x})$ replaces probability $P(y \mid \mathbf{X=x})$ in binary formulae. 
%
We use $y$ to denote $Y=1$,  $x_i $ to represent $X_i=x_i$ and $do(x_i)$ to represent $do(X_i = x_i)$ when the context is clear.
	
	
	

Let $\mathbf{x} = (x_1, x_2, \ldots, x_m)$ represent an individual. For a personalised decision, there needs to estimate the causal effect of each $X_i \in \mathbf{X}$ on $Y$ when other variable values in $\mathbf{x}$ are kept unchanged for all $X_j \ne X_i$. This means that each $X_i$ is alternated as a treatment. The setting is different from normal causal effect estimation problem considering only a fixed treatment. Alternating features as treatments poses a challenge for estimating causal effects. An example is given below.

\begin{figure}[tb]
	\centering
	\begin{subfigure}[b]{0.22\textwidth}
		\centering
		\includegraphics[scale=0.2]{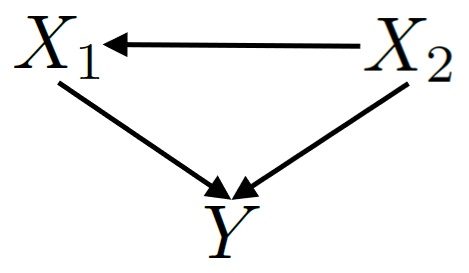}
		\caption{}
	\end{subfigure}
	\begin{subfigure}[b]{0.22\textwidth}
		\centering
		\includegraphics[scale=0.2]{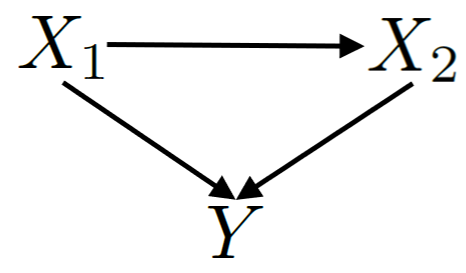}
		\caption{}
	\end{subfigure}
	\caption{Two exemplar causal DAGs.}
	\label{fig_example}
\end{figure}
	
	
	

\begin{example}[A challenge for measuring causal effect when alternating features as treatments] \label{ex_CATEvsCDE}
Consider the causal DAG underlying the data set as in Figure~\ref{fig_example}(a). $X_1$ and $X_2$ are two parents of $Y$.  For individual $\mathbf{x} = (x_1, x_2)$, we aim at estimating causal effect of $X_1$ on $Y$ ($X_1$ is the treatment) when $X_2 = x_2$ is kept unchanged and causal effect of $X_2$ on $Y$ ($X_2$ is the treatment) when $X_1 = x_1$ is kept constant. 


Let us consider $X_1$ as the treatment firstly. $\cate(X_1, Y ; X_2 = x_2) = P(y \mid do(X_1 = 1), x_2) - P(y \mid do(X_1 = 0), x_2))$ describes the conditional average treatment effect of $X_1$ on $Y$  conditioning on $X_2 = x_2$.  CATE has been proposed to describe causal heterogeneity (the average causal effects in different subgroups of the population)~\cite{abrevaya2015estimating,AtheyImbens2016_PNAS,wager2018estimation,Kuenzel2019}. CATE is mostly used for personalised decision making~\cite{zhang2020unified,zhang2017mining} and policy evaluation~\cite{Athey2017_Science}. Some researchers use it to approach individual treatment effect~\cite{Shalit2016,lu2018estimating}.

We now consider $X_2$ as the treatment. CATE cannot be used any more since $X_1$ is not a confounder of $(X_2, Y)$. $X_1$ is a mediator between $(X_2, Y)$. $\cde(X_2, Y; X_1 = x_1) = P(y \mid do(X_2 = 1), do(x_1)) - P(y \mid do(X_2 = 0), do(x_1))$ indicates the controlled direct effect of $X_2$ on $Y$ when $X_1$ is set to a fixed value $x_1$. CDE has been proposed for causal mediation analysis indicating the direct effect of a treatment on the outcome when mediators are kept constant~\cite{robins1992identifiability,DirectPaperR-273,vanderweele2011controlled,vanderweele2016mediation}.  CDE has been used for policy evaluation~\cite{robins1992identifiability,DirectPaperR-273} and political explanations~\cite{acharya2016explaining}.

Two different types of causal effects describe the effects we wish to capture. Their use depends on the underlying causal DAG. When we do not know the underlying DAG, we do not know which one to use. 
\end{example}


The difference between the two types of causal effects is in their conditional sets (apart from the treatment) of the conditional probabilities: $do$ free (observational) for CATE and with $do$ (interventional) for CDE. In general, the interventional probability is different from the observational probability and hence CATE and CDE are different.   

Fortunately, in some situations, CDE and CATE have the same estimate. There needs the following assumption. 

\begin{assumption}[Complete direct causes] All direct causes of the outcome $Y$ are included in $\mathbf{X}$.
	\label{asm_complete-Direct-Causes}
\end{assumption}




We show the consistency between the two types of causal effects in the following observation. We do not assume knowing the causal graph in the following discussions. 

\begin{observation}[CDE and CATE have the same estimate in data] Let a data set include variables $(X_1, X_2,  Y)$ where $X_1$ and $X_2$ are direct causes of $Y$. In data, we observe $X_1 \nindep Y$, $X_2 \nindep Y$ and $X_1 \nindep X_2$ where $\nindep$ means not independent (or associated). From the data, we can infer one of the two DAGs, Figure~\ref{fig_example}(a) or Figure~\ref{fig_example}(b), represents the underlying causal DAG that generates the data set. The edge direction between $X_1$ and $X_2$ cannot be inferred from the data. We aim at estimating the causal effect of $X_1$ on $Y$ when $X_2$ is kept constant at $x_2$. 
	
If Figure~\ref{fig_example}(a) is the underlying causal DAG, we estimate $\cate(X_1, Y; X_2 = x_2)$. If Figure~\ref{fig_example}(b) is the underlying causal  DAG, we estimate $\cde(X_1, Y; X_2 = x_2)$. When $X_1$ and $X_2$ are direct causes of $Y$, $\cate(X_1, Y; X_2 = x_2)$ in DAG(a) has the same value as $\cde(X_1, Y; X_2 = x_2)$ in DAG(b). Both estimates are $P(y \mid X_1 = 1,  x_2) - P(y \mid X_1 = 0,  x_2)$. 
\end{observation}


\begin{proof}	
Let $\mathcal{G}$ be a DAG, $\mathcal{G}_{\underline{X_1}}$ be a DAG by removing from $\mathcal{G}$ outgoing edges from $X_1$, and  $\mathcal{G}_{\overline{X_1}\underline{X_2}}$ be a DAG by removing from $\mathcal{G}$ incoming edges to $X_1$ and outgoing edges from $X_2$. In the following deduction, Assumption~\ref{asm_complete-Direct-Causes} is assumed. 

Assume that DAG (a) is the underlying causal DAG $\mathcal{G}$. We have the following deduction.  Rules refer those in Theorem 3.4.1~\cite{Pearl2009_Book}. 
\begin{eqnarray*}
	&P(y \mid do(x_1), x_2)	 & \\
	&= P(y \mid x_1, x_2) &  \because Y \indep X_1 \mid X_2~in~\mathcal{G}_{\underline{X_1}}~(Rule~2)
\end{eqnarray*}

Assume that DAG (b) is the underlying causal DAG $\mathcal{G}$. We have the following deduction.
\begin{eqnarray*}
&P(y \mid do(x_1), do(x_2))	 & \\
&= P(y \mid do(x_1), x_2) &  \because Y \indep X_2 \mid X_1~in~\mathcal{G}_{\overline{X_1}\underline{X_2}}~(Rule~2)  \\
&= P(y \mid x_1, x_2) &  \because Y \indep X_1 \mid X_2~in~\mathcal{G}_{\underline{X_1}}~(Rule~2)
\end{eqnarray*}

Therefore,  $\cate(X_1, Y; X_2 = x_2) = P(y \mid X_1 = 1, x_2) - P(y \mid X_1 = 0, x_2)$ and $\cde(X_1, Y; X_2 = x_2) = P(y \mid X_1 = 1,  x_2) - P(y \mid X_1 = 0,  x_2)$. Both have the same value. 
\end{proof}

      
The above observation shows that, when all direct causes are conditioned, intervention (with $do$) and observation ($do$-free) are exchangeable. We have the following lemma to show this. Let $\pa(Y)$ denote all the parents (direct causes) of $Y$ in the underlying causal DAG.

\begin{lemma}\label{lemma-Exchangability}[Exchangeability of intervention and observation] Under Assumptions~\ref{asm_Markovcondition}, \ref{asm_Faithfulness} and~\ref{asm_complete-Direct-Causes}, let $\pa(Y) = \{X_1, X_2,  \ldots, X_k,  \ldots,X_p\}$ where the index of variables is arbitrary ($X_1$ is assumed the treatment) and $k$ is a value in $[2, p]$. $P(y \mid do(x_1), do(x_2), \ldots do(x_k), x_{k+1},  \ldots, x_p) = P(y \mid x_1, x_2, \ldots x_k, x_{k+1}, \ldots, x_p)$. 
	\end{lemma}


\begin{proof}
We first show how to reduce $do(x_k)$ to $x_k$.
Let $\mathcal{G}$ be the causal DAG underlying data, and $\mathcal{G}_{\overline{X_1}, \overline{X_2}, \ldots, \overline{X_{k-1}}, \underline{X_k}}$ be a DAG removing from $\mathcal{G}$ all incoming edges into $X_1, X_2, \ldots, X_{k-1}$ and outgoing edges from $X_k$.  Rules refer those in Theorem 3.4.1~\cite{Pearl2009_Book}. 
\begin{eqnarray*}
	&P(y \mid do(x_1), do(x_2), \ldots do(x_k), x_{k+1},  \ldots, x_p)	 & \\
	&= P(y \mid do(x_1), do(x_2), \ldots do(x_{k-1}), x_k, x_{k+1},  \ldots, x_p) &  \\
	& \because Y \indep X_k \mid x_1, x_2, \ldots x_{k-1}, x_{k+1}, \ldots, x_p & \\
	&in~\mathcal{G}_{\overline{X_1}, \overline{X_2}, \ldots, \overline{X_{k-1}}, \underline{X_k}}~(Rule~2) &
\end{eqnarray*}

Repeat the above deduction for $X_{k-1}, \ldots, X_2, X_1$ respectively, we obtain. $P(y \mid do(x_1), do(x_2), \ldots, do(x_{k-1}), \\x_k, \ldots, x_p) = P(y \mid x_1, x_2, \ldots x_{k-1}, x_k, \ldots, x_p)$. 

The lemma is proven. 
	\end{proof}

Using $do$ or $do$-free in the conditional set for causal effect estimation is determined  by the underlying causal DAG, specifically, paths among parents of $Y$. Recall our discussions in Example~\ref{ex_CATEvsCDE}. When $X_1$ is the treatment, $P(y \mid do(x_1), do(x_2))$  is used since $X_2$ is a descendant of $X_1$. 

Lemma~\ref{lemma-Exchangability} identifies a case where paths  among parents do not affect CATE and CDE estimation since $do$ and $do$-free are exchangeable. Lemma~\ref{lemma-Exchangability} implies that CATE and CDE estimation within $\pa(Y)$ are consistent. In other words, even without knowing the paths among direct causes of $Y$, the causal effects estimated based on CATE and CDE are identical when all direct causes are known and used. 





In the above discussions, $\mathbf{X}$ includes only the parents of $Y$. Let us consider the general case where $\mathbf{X}$ includes all variables. The conditional set of CATE does not allow the descendants of $Y$ and hence CATE is not applicable for the general case. In Definition~\ref{def:CDE}, the controlled variables of CDE are restricted to mediators since the motivation is for mediation analysis. If the controlled variables are relaxed to all variables in the system except the treatment, $\cde$ can be explained in a hypothetical experiment setting~\cite{Pearl2009_Book}. When all other variables are set as $\mathbf{X}{\backslash X_i} = \mathbf{x'}$, and only $X_i$ is changed from 0 to 1, $\cde$ of $X_i$ on $Y$ is measured as the change of $Y$, i.e., $\cde(X_i, Y; \mathbf{X} \backslash X_i = \mathbf{x'}) = P(y \mid do (X_i=1), do (\mathbf{x'})) - P(y \mid do (X_i=0), do (\mathbf{x'}))$.

We have the following lemma on the identifiability of CDE. 

\begin{lemma}[Identifiability of CDE] \label{lemma_DCE-Identifibility} 
Let $\{\mathbf{X}, Y\}$ include all variables in the system, and $\mathbf{X'} = \mathbf{X} \backslash X_i$.  $\mathbf{x'}$ is a value of $\mathbf{X'}$.	
Under Assumptions \ref{asm_Markovcondition}, \ref{asm_Faithfulness} and~\ref{asm_complete-Direct-Causes}. For each $X_i \in \pa(Y)$ and $\mathbf{Z} = \pa(Y) \backslash X_i$,  $\cde(X_i, Y;  \mathbf{X'=x'}) = P(y \mid X_i=1, \mathbf{Z=z}) - P(y \mid X_i=0, \mathbf{Z=z})$.  For each $X_i \notin \pa(Y)$, $\cde(X_i, Y;  \mathbf{X'=x'}) = 0$. 
 \end{lemma} 

\begin{proof}
Firstly, let $\mathbf{X = \{X_A, X_D, X_I\}}$ where $\mathbf{X_A}$ includes all ancestors of $Y$, $\mathbf{X_D}$ includes all descendants of $Y$, and $\mathbf{X_I}$ includes all other variables (i.e. a variable $I \in \mathbf{X_I}$ is not linked to $Y$ via a path, $I$ and $Y$ share a same descendant, or $I$ and $Y$ share a same ancestor.).    	

Let $\mathcal{G}$ be the causal DAG underlying data, $\mathcal{G}_{\overline{X_i}}$ be the DAG removing from $\mathcal{G}$ all incoming edges into $X_i$, and $\mathcal{G}_{\overline{X_i}, \overline{\mathbf{W}}}$ be the DAG removing from $\mathcal{G}$ all incoming edges into $X_i$ and into every variable in $\mathbf{W}$.  Rules refer to those in Theorem 3.4.1~\cite{Pearl2009_Book}. 

Let $\mathbf{X'_A} = \mathbf{X_A} \backslash X_i$ and $\mathbf{x'_A}$ be a value of $\mathbf{X'_A}$.  $\mathbf{x_D}$ and  $\mathbf{x_I}$ are values of $\mathbf{X_D}$ and $\mathbf{X_I}$ respectively. $\{\mathbf{x'_A}, \mathbf{x_D}, \mathbf{x_I}\} = \mathbf{x'}$ where $\mathbf{x'}$ is a value of $\mathbf{X'} = \mathbf{X} \backslash X_i$.  
\begin{eqnarray*}
	&P (y \mid do(x_i), do(\mathbf{x'})) & \\
	&= P(y \mid do(x_i), do(\mathbf{x'_A}), do(\mathbf{x_D}), do(\mathbf{x_I}))	 & \\
	&= P(y \mid do(x_i), do(\mathbf{x'_A}), do(\mathbf{x_D}))	   & \\
	& \because Y \indep \mathbf{X_I} \mid X_i, \mathbf{X'_A}, \mathbf{X_D}~in~\mathcal{G}_{\overline{X_i}, \overline{\mathbf{X'_A}}, \overline{\mathbf{X_D}}}~(Rule~3) & \\
	&= P(y \mid do(x_i), do(\mathbf{x'_A}))	   & \\
& \because Y \indep \mathbf{X_D} \mid X_i, \mathbf{X'_A}~in~\mathcal{G}_{\overline{X_i}, \overline{\mathbf{X'_A}}}~(Rule~3) & 
\end{eqnarray*}

Based on the Markov Assumption~\ref{asm_Markovcondition}, for any $X_j \notin \pa(Y)$ and $X_j \in \mathbf{X'_A}$, $X_j \indep Y \mid \pa(Y)$. Therefore, for $X_i \in \pa(Y)$, $P(y \mid do (x_i), do (\mathbf{X_A' = x_A'})) = P(y \mid do (x_i), do(\mathbf{Z= z})$.

According to Lemma~\ref{lemma-Exchangability}, $P(y \mid do (x_i), do(\mathbf{Z= z}) = P(y \mid x_i,  \mathbf{Z= z})$.

Therefore, $\cde(X_i, Y;  \mathbf{X'=x'}) = P(y \mid X_i=1, \mathbf{Z=z}) - P(y \mid X_i=0, \mathbf{Z=z})$. 

To prove for each $X_i \notin \pa(Y)$, $\cde(X_i, Y;  \mathbf{X'=x'}) = 0$. we consider three cases, $X_j \in \mathbf{X_I}$, $X_j \in \mathbf{X_A}$, and $X_j \in \mathbf{X_D}$. In all three cases, we have the following deduction. 
\begin{eqnarray*}
	&P (y \mid do(x_i), do(\mathbf{X' = x'})) = P (y \mid do(\mathbf{X' = x'})) & \\
	& \because Y \indep X_i \mid  \mathbf{X'}~in~\mathcal{G}_{\overline{X_i}, \overline{\mathbf{X'}}}~(Rule~3) & 
\end{eqnarray*}
Therefore, $\cde(X_i, Y;  \mathbf{X'=x'}) = 0$

The Lemma is proven.
%
	\end{proof}

Lemma~\ref{lemma_DCE-Identifibility} shows that CDE of features on $Y$ can be estimated from data if the set of direct causes (or parents) of  $Y$ is known.  Only a parent of $Y$ has a non-zero CDE on $Y$ and hence we will focus on CDEs of $\pa(Y)$.

In some data sets, especially with high dimensionality, knowing direct causes may not be easy. A number of methods have been proposed to learn concepts from predictive models and data~\cite{kim2018interpretability,goyal2019explaining,yao2022concept}. Raw features can be generalised to concepts that are easy to be understood by users. Using the concepts, the direct causes are easier identified than using raw features.

\subsection{Identifying direct causes in data}
\label{sec-identifyingDirectCauses}

When the parents of $Y$ are unknown, they will need to be found in the data. The following assumption is needed to find the parents of $Y$ in data. 

\begin{assumption}[No descendants of $Y$] $\mathbf{X}$ contains no descendants of $Y$.  Or, $\mathbf{X}$ contains no variables affected by the outcome $Y$. 
	\label{asm_pretreatment}
\end{assumption}


From the data, only a Complete Partially Directed Acyclic (CPDAG)~\cite{Spirtes2010_Intro,Spirtes2000} can be learned. A CPDAG is a graph representing an equivalence class of DAGs that encode the same conditional independencies. An example of CPDAGs is shown in Figure~\ref{fig_Equivalence}. 
%
In general, it is impossible to learn a unique DAG from data since from data alone it is impossible to orient all the edges. Note that, with some other assumptions, such as linear functions and non-Gaussian noises~\cite{shimizu2006linear} or additive non-linear Gaussian noises~\cite{hoyer2008nonlinear}, causal directions can be identified. In this paper, we do not make such assumptions.  

\begin{figure}[t]
	\centering
	\includegraphics[scale=0.2]{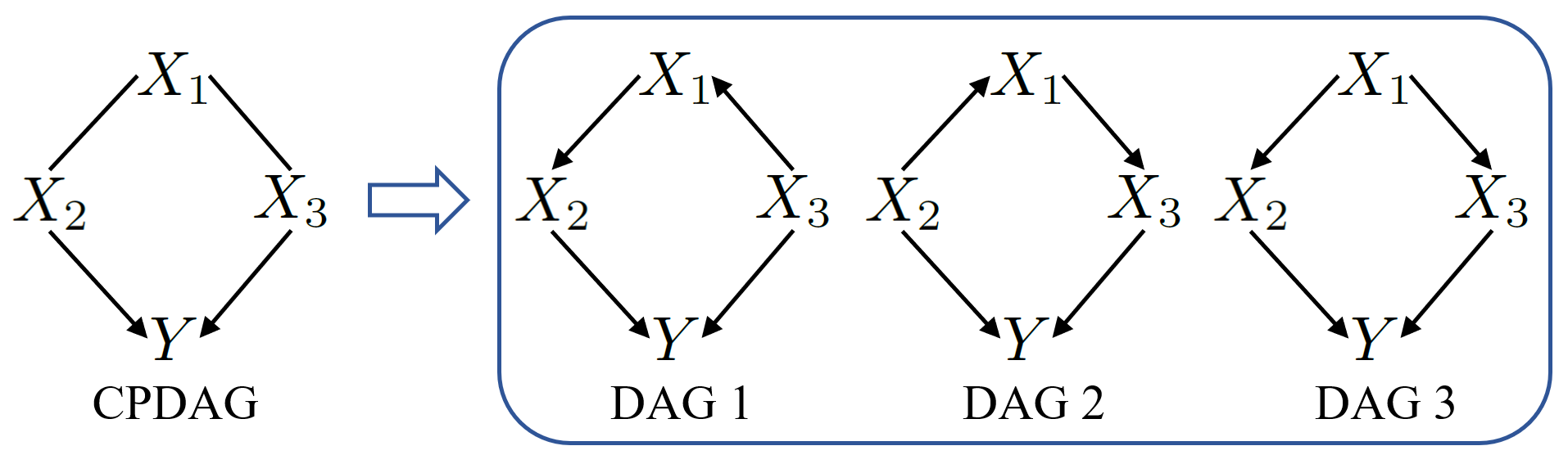}
	\caption{A CPDAG (left) represents an equivalence class of DAGs (right). Only the CPDAG is identifiable from data.}
	\label{fig_Equivalence}
\end{figure}

In practice, we do not assume that there is a known underlying causal DAG or that a unique DAG can be learned from the data. However, given the assumptions, the direct causes can be identified in data.  

\begin{lemma}[Identifiability of direct causes of $Y$ in data]~\label{lemma-CauseIdentifibility} Under Assumptions \ref{asm_Markovcondition}, \ref{asm_Faithfulness}, \ref{asm_complete-Direct-Causes}, \ref{asm_pretreatment}, and there are no statistical errors in conditional independence tests, the direct causes of $Y$ (i.e., $\pa(Y)$) can be identified in data. 
\end{lemma} 


\begin{proof}
With Assumptions~\ref{asm_Markovcondition} and~\ref{asm_Faithfulness}, there will be no any other variables $d$-separating a parent node from $Y$. When there are no statistical errors in conditional independence tests, a parent node, say $X_i$, will be dependent with $Y$ conditioning on any subset of $\mathbf{X} \backslash X_i$. With Assumption~\ref{asm_pretreatment}, no other variable (apart from parents) has such a property since the variable will be independent of $Y$ given the set of parents according to the Markov condition. Therefore, the set of parents can be identified using conditional independence tests.  

	%

\end{proof}

\begin{algorithm}[tb]
	\caption{Building a \textbf{M}odel for \textbf{O}utcome prediction and C\textbf{D}E \textbf{E}stimation (MODE)} 	\label{alg_What-if}
	{\textbf Input}: a data set generated from the underlying causal DAG with variables $(\mathbf{X}, Y)$ satisfying the Assumption~\ref{asm_complete-Direct-Causes} and~\ref{asm_pretreatment}.  An instance $\mathbf{x}$ for what-if questions. Parameters $k$ and $\delta$. \\
	{\textbf Output}: Predicted outcome for $\mathbf{x}$ and top $k$ features with the highest $\cde$s. 
	
	
	\begin{algorithmic}[1]
		\STATE {Discover $\pa(Y)$ in $\mathbf{X}$.}
		\STATE {Project the data set to $(\pa(Y), Y)$.}
		\STATE {Build a model $f(\pa(Y))$ using a machine learning method within the projected data set.} 
		\STATE {$y_{\mathbf{x}} = f(\mathbf{x'})$~~$//$ Outcome prediction for $\mathbf{x}$. where $\mathbf{x'}$ includes values of $\mathbf{x}$ within $\pa(Y)$}
		\FOR {each $X_i \in \pa(Y)$}
		\STATE {Let $\mathbf{Z} = \pa(Y) \backslash X_i$ and $(x_i, \mathbf{z})$ contain the values of $\mathbf{x}$ within $\pa(Y)$.} \\
		\COMMENT {Case 1: Binary treatment and outcome}\\  			
		\STATE {Let $P_f (y \mid x_i, \mathbf{z})$ be the probability of $f(\pa(Y))$} classifying instance $(x_i, \mathbf{z})$ as $y$, the class of interest. 		
		\STATE $\cde_{x_i} (X_i, Y; \mathbf{z}) = P_f (y \mid \overline{x}_i, \mathbf{z}) - P_f (y \mid x_i, \mathbf{z})$ \\
		\COMMENT {Case 2: Continuous treatment and binary outcome}\\  			
		\STATE {$\cde_{x_i} (X_i, Y; \mathbf{z}) = P_f (y \mid x_i + \delta, \mathbf{z}) - P_f (y \mid x_i, \mathbf{z})$} \\
		\COMMENT {Case 3: Binary treatment and continuous outcome}\\  						
		\STATE {Let $f(x_i, \mathbf{z})$ be the output of $f(\pa(Y))$} at input $(x_i, \mathbf{z})$. 		
		\STATE {$\cde_{x_i} (X_i, Y; \mathbf{z}) = f(\overline{x}_i, \mathbf{z}) - f(x_i, \mathbf{z})$} \\
		\COMMENT {Case 4: Continuous treatment and continuous outcome}\\  						
		\STATE {$\cde_{x_i} (X_i, Y;  \mathbf{z}) = f(x_i + \delta, \mathbf{z}) - f(x_i, \mathbf{z})$}
		\ENDFOR
		\STATE {Sort $\cde_{x_i} (X_i, Y;  \mathbf{z})$ for all $X_i \in \pa(Y)$}
		\STATE Return $y_{\mathbf{x}}$ and top $k$ parents of $Y$ with the highest $\cde$s.
	\end{algorithmic}
\end{algorithm}

\subsection{Linking a predictive model for causal effect estimation}
\label{sec_model-buildinbg}
		
CDEs can be estimated by conditional probabilities when the set of parent of $Y$ is known. A predictive model can model conditional probabilities very well, and hence can be used to estimate CDEs. To estimate CDEs using a predictive model, we make the following assumption.



\begin{assumption}[Consistent model]\label{asm-Perfect} Model $Y = f(\pa(Y))$ is trained in an unbiased sample of a population. $P(Y \mid \pa(Y))$ obtained from the model is the same as the conditional probability of $Y$ in the population. 
\end{assumption}

With Assumption~\ref{asm-Perfect}, conditional probability read from a predictive model is the  conditional probability in the population. 

Now, we can build a predictive model for estimating CDE of each feature in the model on the outcome $Y$ and for predicting the outcome given an input instance. The proposed algorithm, MODE, is shown in Algorithm~\ref{alg_What-if}. The model makes an outcome prediction for an input instance and provides a list of $k$ potential treatment variables for the individual to change the outcome with the highest effects.      

Lines 1 and 2 of MODE discover parents of $Y$ and project the data set to the parents of $Y$ and $Y$.   A local causal structure learning algorithms, such as Hiton PC~\cite{Aliferis:2010}, MMPC~\cite{Aliferis:2010} and PC-Select~\cite{PC-select-2010} can be used to learn $\pc(Y)$, the set of all Parents and Children of $Y$, in data. Then with Assumptions~\ref{asm_complete-Direct-Causes} and~\ref{asm_pretreatment}, $\pa(Y) = \pc(Y)$.

Lines 3 and 4  build a predictive model $f(\pa(Y))$ using the features including parents of $Y$ and $Y$ and predict the outcome for input instance $\mathbf{x}$. Any supervised machine learning model can be used to build the model. The usage of the model for outcome prediction is the same as with a normal predictive model. 

Lines 5-11 estimate the CDE for every feature in $\pa(Y)$ on $Y$ using Lemma~\ref{lemma_DCE-Identifibility} for various variable types of the feature and outcome. The conditional probabilities used in CDE estimation are obtained from the predictive model $f(\pa(Y))$. 

The estimated $\cde$ can be used for answering a what-if question: what $Y$ will change if $X_i = x_i$ is changed to $ \overline{x}_i$, i.e., the opposite value of $x_i$ (or $X_i = x_i + \delta$ for a continuous treatment) given other values of $\mathbf{x}$ are kept constant?

In the causal inference convention, $X_i = 0$ represents no treatment (control) and $X_i = 1$ indicates treatment (treated). In a what-if question, the control is the given value $x_i$ and the treated is the changed value (e.g. to the opposite value of $x_i$ ). Therefore, in the algorithm,  $\cde$ includes a subscript indicating the control $x_i$. 

Lines 14 and 15 sort the features in $\pa(Y)$ based on their CDEs and output the predicted outcome of $\mathbf{x}$ and the top $k$ features with the highest CDEs.

When the assumptions are satisfied, a trained predictive model can be used for $\cde$ estimation without accessing the data (data is only needed in the training stage) in a way similar to the outcome prediction for an input instance.

\subsection{A robustness analysis of a causally interpretable model}
\label{sec_robustness}

\begin{figure}[t]
	\centering
	\begin{subfigure}[b]{0.30\textwidth}
		\centering
		\includegraphics[scale=0.22]{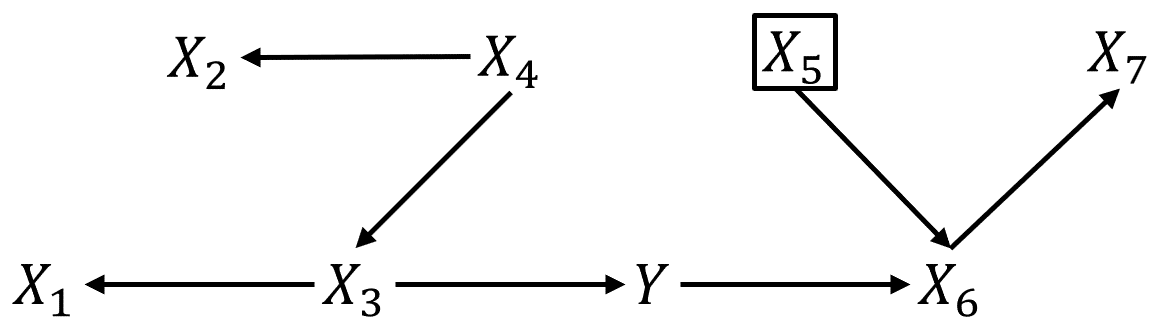}
		\caption{}
	\end{subfigure}
	\begin{subfigure}[b]{0.30\textwidth}
		\centering
		\includegraphics[scale=0.22]{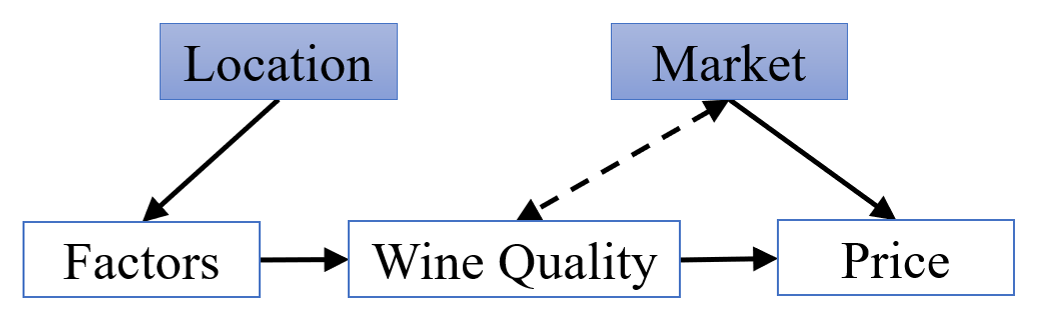}
		\caption{}
	\end{subfigure}
	\caption{(a) An example shows different types of features that are associated with $Y$ in a causal graph.  $X_5$ in the box is not associated with $Y$. However, when conditioned on $X_6$, $X_5$ becomes associated with $Y$. This association is called a spurious association. (b) An example shows how a spurious association is used in a predictive model. Gradient filled boxes represent unobserved variables.}
	\label{fig_winequality}
\end{figure}

Previous two subsections show that a predictive model built on all direct causes (parents) can be used to estimate $\cde$ of any direct cause on $Y$. A model built on all direct causes are hence causally interpretable. Some readers might think that its performance is inferior to that of a predictive model built on all features associated with $Y$ including descendants of $Y$. We will analyse this through another aspect of a predictive model -- robustness.  

A model built on $\pa(Y)$ will be more robust than a model built on all features associated with $Y$. The robustness of a model means the capability of the model to make reliable predictions in a new environment.

In a DAG of features set $\mathbf{X}$ and an outcome $Y$, three types of features will be associated with $Y$: 1) the ancestors of $Y$ including the parents of $Y$; 2) the descendants of $Y$; and 3) the descendant nodes  of the ancestors of $Y$. We show the three types of variables in Figure~\ref{fig_winequality}(a). $X_3$ and $X_4$ belong to type 1; $X_6$ and $X_7$ belong to type 2; and $X_1$ and $X_2$ belong to type 3. When $\pa(Y)$ is used to build a predictive model, other ancestors of $Y$ apart from $\pa(Y)$ and the descendant nodes of all ancestors of $Y$ (type 2) do not help building the model since they are independent of $Y$ given $\pa(Y)$ based on the Markov condition (Assumption~\ref{asm_Markovcondition}). Therefore, we will discuss in the following whether we should use the descendants of $Y$ for building a predictive model.  

When using the descendants of $Y$ (type 3) for building a predictive model, other variables that are not associated with $Y$ will be included in the model and this may lead to non-robustness of the model. 

We first explain how a spurious association is produced by including a descendant variable of $Y$ in a model.  Let us look at the DAG in Figure~\ref{fig_winequality}(a), where $X_6$ is a child node of $Y$. $X_5$ is a parent node of $X_6$. Edges $Y \to X_6$ and $X_5 \to X_6$ form a collider at $X_6$, and hence $Y$ and $X_5$ are associated when conditioning on $X_6$. This association is spurious since $X_5$ and $Y$ are actually independent. So, when $(X_3, X_6)$ are used to build a model, $X_5$ is unintentionally included in the model.

We use an example to demonstrate the impact of a spurious association on the model robustness.  An example including a spurious association in a predictive model is given in Figure~\ref{fig_winequality}(b). $Y$ is Wine Quality.  Factors are parent of $Y$,  Price is the child node of $Y$, Location is an ancestor of $Y$ (not a parent), and Market is a parent node of Price . Both Location and Market are unobserved. Price is associated with Wine Quality strongly and hence seems to help predict Wine Quality. However, Market also affects Price. When Price is used to predict Wine Quality, Market is unintentionally introduced into the model because of the spurious association. For a bottle of wine with the same quality, its price can be quite different in two markets. So a model built in a data set from a market including Price as a feature may not work in another data set from a different market, and therefore the model is not robust.

The lack of robustness of a model cannot be tested in the model building process using normal cross validation. For example, if the data for building a model is collected in one market, the market is the same for both training and test data sets. Then the model performance will be consistent in both training and test data sets. However, the model will fail when it is used in a different market.

In contrast, a model built using all direct causes will not be affected by the unobserved variables and is hence robust. Location is an unobserved cause of Factors. Location affects Wine Quality by changing Factors and the changes due to the location are reflected in Factors.  When conditioning on Factors, Location is independent of  Wine quality and hence does not contribute to predicting $Y$. Therefore, a model using Factors will give consistent predictions in different locations.

In summary, we have the following conclusion. We say that an unobserved variable $U$ is involved in model $Y = f(\mathbf{X})$ if $U \indep Y$ but $U \nindep Y \mid \mathbf{X'}$ where $\mathbf{X'} \subseteq \mathbf{X}$. An unknown variable set $\mathbf{U}$ is involved  in model $Y = f(\mathbf{X})$ if some $U \in \mathbf{U}$ are involved in the model.   

\begin{lemma}[Involvement of unobserved variables in a predictive model]\label{lemma_Involovement}
	Let the feature set $\mathbf{X}$ include $\{\mathbf{X_A, X_D,  X_I}\}$ that are associated with $Y$. In the causal graph underlying the data set, $(\mathbf{X_A, X_D, X_I})$ consist of the ancestors of $Y$, the descendants of $Y$, and the descendants of the ancestors of $Y$ respectively. $\{\mathbf{X_A, X_D,  X_I}\}$ have unobserved parents $\{\mathbf{U_A, U_D,  U_I}\}$ respectively. Let $\mathbf{X_A}$ contains all direct parents of $Y$. $\mathbf{X_I}$ and $(\mathbf{X_A} \backslash \pa(Y))$ (`$\backslash$' denotes set difference) are ignored when $\pa(Y)$ is used in a predictive model. $\mathbf{U_D}$ is involved in model $Y = f(\mathbf{\pa(Y), X_D})$, but $\mathbf{U_A}$, $\mathbf{U_D}$ or $\mathbf{U_I}$  is not involved in model $Y = f(\mathbf{\pa(Y)})$.      
\end{lemma}


\begin{proof}
	
Based on the Markov condition (Assumption~\ref{asm_Markovcondition}), $\mathbf{X_I} \indep Y | \pa(Y)$ and $(\mathbf{X_A} \backslash \pa(Y)) \indep Y \mid \pa(Y)$. Therefore,  $\mathbf{X_I}$ and $(\mathbf{X_A} \backslash \pa(Y))$  are ignored in building a predictive model that includes $\pa(Y)$.
	
$f(\mathbf{X_A, X_D})$ uses conditional probability $P(y \mid (\mathbf{X_A, X_D}))$ for predictions. In data, $\mathbf{U_D} \indep Y$. Directed paths from $Y$ to $\mathbf{X_D}$ and directed paths from $\mathbf{U_D}$ to $\mathbf{X_D}$ form colliders at nodes $\mathbf{X_D}$. Hence,  $\mathbf{U_D} \nindep Y \mid \mathbf{X_D}$. Therefore, $\mathbf{U_D}$ is involved in model $f(\mathbf{X_A, X_D})$.

$f(\pa(Y))$ uses conditional probability $P(y \mid (\pa(Y))$ for predictions. $\mathbf{U_A} \indep Y \mid \pa(Y)$ and $\mathbf{U_I} \indep Y  \mid \pa(Y)$  based on the Markov condition (Assumption~\ref{asm_Markovcondition}).  Hence, $\mathbf{U_A}$ and $\mathbf{U_I}$ are not involved in model $f(\pa(Y))$. Furthermore, $\mathbf{U_D} \indep Y \mid \pa(Y)$, and hence $\mathbf{U_D}$ is not involved in model $f(\mathbf{X_A})$.
	
The Lemma is proven. 	
\end{proof}

Lemma~\ref{lemma_Involovement} indicates the predictive model using the parents of $Y$ only will not involve other unobserved variables in the system, and hence a model is robust to the unobserved changes. In contrast, using descendant variables of the outcome in a model risks involving unobserved variables (spurious associations) in the model, and may lead to underperformance of the predictive model in a new environment. 

In causal feature selection literature~\cite{yu2020causality,YuUnifiedFeatureSelection2021}, the Markov blanket is considered the optimal feature set for classification. Markov blanket of the outcome $Y$ includes the parents (direct causes) of $Y$ and the children (variables directly affected by the outcome) of $Y$, and the parents of every child of $Y$. Predictive models built on Markov blanket have been shown robust against the environment changes~\cite{yang2023learning}. When all variables in the Markov blanket are observed, a model built on the Markov blanket is accurate and robust. We consider a general case, where some parents of some children of $Y$ may not be observed (such as the Market variable in Figure~\ref{fig_winequality}(b)). In this case, using variables affected by the outcome will lead to the non-robustness of a model. Note that unobserved variables are very common in practice, and hence not using variables affected by the outcome in model building is a sound suggestion.

\begin{table*}[t]
	\setlength\tabcolsep{1.0pt}
	{\footnotesize \begin{tabular}{ccccccc|ccccc}
			\toprule
			1 & MLP         & GBR         & LR          & DT          & RF          & SGD         & L-DML      & K-DML      & CF          & L-DRL       & F-DRL       \\ \midrule
			\multirow{2}{*}{2k}  & 30.5 $\pm$ 10.7 & 34.2 $\pm$ 11.6 & 30.2 $\pm$ 12.1 & 35.5 $\pm$ 13.8 & 33.9 $\pm$ 11.7 & 32.2 $\pm$ 11.1 & 26.2 $\pm$ 8.1 & 38.3 $\pm$ 10.0  & 30.9 $\pm$ 12.4 & 30.0 $\pm$ 10.7   & 31.5 $\pm$ 11.2 \\
			& \cellcolor{lightgray}11.1 $\pm$ 9.9  & \cellcolor{lightgray}12.3 $\pm$ 10.5 & \cellcolor{lightgray}12.3 $\pm$ 8.7  & \cellcolor{lightgray}13.5 $\pm$ 11.5 & \cellcolor{lightgray}14.0 $\pm$ 10.6   & \cellcolor{lightgray}12.9 $\pm$ 9.7  & \cellcolor{lightgray}8.2 $\pm$ 6.0    & \cellcolor{lightgray}16.0 $\pm$ 9.8   & \cellcolor{lightgray}12.0 $\pm$ 12.4   & \cellcolor{lightgray}11.2 $\pm$ 10.8 & \cellcolor{lightgray}13.2 $\pm$ 8.5  \\ \midrule
			\multirow{2}{*}{4k}  & 32.3 $\pm$ 11.0   & 28.5 $\pm$ 7.6  & 28.6 $\pm$ 7.9  & 29.9 $\pm$ 7.3  & 30.0 $\pm$ 8.5    & 31.6 $\pm$ 9.7  & 30.1 $\pm$ 8.2 & 34.7 $\pm$ 6.5 & 32.3 $\pm$ 7.8  & 27.3 $\pm$ 9.1  & 29.9 $\pm$ 7.6  \\
			& \cellcolor{lightgray}9.6 $\pm$ 6.5   & \cellcolor{lightgray}7.1 $\pm$ 5.0     & \cellcolor{lightgray}7.4 $\pm$ 5.1   & \cellcolor{lightgray}8.0 $\pm$ 6.2     & \cellcolor{lightgray}8.8 $\pm$ 7.1   & \cellcolor{lightgray}10.2 $\pm$ 6.5  & \cellcolor{lightgray}9.2 $\pm$ 7.5  & \cellcolor{lightgray}8.1 $\pm$ 6.3  & \cellcolor{lightgray}9.1 $\pm$ 5.3   & \cellcolor{lightgray}10.2 $\pm$ 6.5  & \cellcolor{lightgray}7.0 $\pm$ 4.5     \\ \midrule
			\multirow{2}{*}{6k}  & 29.4 $\pm$ 6.2  & 27.6 $\pm$ 5.2  & 30.6 $\pm$ 8.3  & 29.4 $\pm$ 7.2  & 30.4 $\pm$ 8.0    & 29.4 $\pm$ 8.7  & 28.0 $\pm$ 7.7   & 33.3 $\pm$ 5.5 & 27.9 $\pm$ 6.7  & 28.1 $\pm$ 6.1  & 30.2 $\pm$ 9.2  \\
			& \cellcolor{lightgray}7.2 $\pm$ 6.0     & \cellcolor{lightgray}5.5 $\pm$ 3.8   & \cellcolor{lightgray}8.2 $\pm$ 6.6   & \cellcolor{lightgray}7.0 $\pm$ 6.0       & \cellcolor{lightgray}8.8 $\pm$ 6.5   & \cellcolor{lightgray}6.8 $\pm$ 6.0     & \cellcolor{lightgray}6.8 $\pm$ 5.9  & \cellcolor{lightgray}7.3 $\pm$ 7.0    & \cellcolor{lightgray}6.9 $\pm$ 5.0     & \cellcolor{lightgray}6.7 $\pm$ 5.3   & \cellcolor{lightgray}7.5 $\pm$ 6.1   \\ \midrule
			\multirow{2}{*}{8k}  & 29.0 $\pm$ 5.9    & 29.9 $\pm$ 6.2  & 29.9 $\pm$ 6.9  & 30.2 $\pm$ 6.0    & 29.4 $\pm$ 4.6  & 31.1 $\pm$ 7.5  & 29.7 $\pm$ 5.8 & 32.8 $\pm$ 5.5 & 29.3 $\pm$ 5.7  & 31.7 $\pm$ 4.4  & 31.0 $\pm$ 5.3    \\
			& \cellcolor{lightgray}6.8 $\pm$ 5.8   & \cellcolor{lightgray}5.7 $\pm$ 4.4   & \cellcolor{lightgray}7.0 $\pm$ 4.7     & \cellcolor{lightgray}5.5 $\pm$ 4.2   & \cellcolor{lightgray}5.2 $\pm$ 3.3   & \cellcolor{lightgray}8.8 $\pm$ 7.7   & \cellcolor{lightgray}6.8 $\pm$ 4.7  & \cellcolor{lightgray}6.4 $\pm$ 5.5  & \cellcolor{lightgray}5.7 $\pm$ 3.9   & \cellcolor{lightgray}4.7 $\pm$ 4.4   & \cellcolor{lightgray}5.8 $\pm$ 5.2   \\ \midrule
			\multirow{2}{*}{10k} & 29.2 $\pm$ 5.4  & 30.6 $\pm$ 4.9  & 29.8 $\pm$ 7.1  & 29.6 $\pm$ 5.8  & 29.7 $\pm$ 5.1  & 30.5 $\pm$ 6.4  & 28.7 $\pm$ 6.4 & 31.5 $\pm$ 4.6 & 29.0 $\pm$ 6.6    & 30.0 $\pm$ 4.6    & 28.7 $\pm$ 5.6  \\
			& \cellcolor{lightgray}5.1 $\pm$ 4.4   & \cellcolor{lightgray}4.4 $\pm$ 3.6   & \cellcolor{lightgray}7.3 $\pm$ 4.8   & \cellcolor{lightgray}4.7 $\pm$ 3.8   & \cellcolor{lightgray}5.8 $\pm$ 4.1   & \cellcolor{lightgray}5.6 $\pm$ 5.4   & \cellcolor{lightgray}5.7 $\pm$ 3.9  & \cellcolor{lightgray}5.7 $\pm$ 4.1  & \cellcolor{lightgray}6.4 $\pm$ 4.9   & \cellcolor{lightgray}4.1 $\pm$ 3.8   & \cellcolor{lightgray}4.6 $\pm$ 3.9   \\ \midrule
			\multirow{2}{*}{20k} & 30.6 $\pm$ 4.0    & 28.7 $\pm$ 4.1  & 29.8 $\pm$ 4.2  & 28.7 $\pm$ 3.5  & 29.2 $\pm$ 3.1  & 28.7 $\pm$ 6.7  & 29.9 $\pm$ 3.9 & 31.5 $\pm$ 3.4 & 29.9 $\pm$ 3.3  & 30.0 $\pm$ 3.5    & 29.2 $\pm$ 3.6  \\
			& \cellcolor{lightgray}5.4 $\pm$ 3.4   & \cellcolor{lightgray}4.0 $\pm$ 2.5     & \cellcolor{lightgray}3.8 $\pm$ 3.5   & \cellcolor{lightgray}3.1 $\pm$ 3.4   & \cellcolor{lightgray}3.4 $\pm$ 2.3   & \cellcolor{lightgray}5.4 $\pm$ 3.4   & \cellcolor{lightgray}3.8 $\pm$ 2.5  & \cellcolor{lightgray}4.4 $\pm$ 3.6  & \cellcolor{lightgray}3.2 $\pm$ 2.6   & \cellcolor{lightgray}4.0 $\pm$ 2.8     & \cellcolor{lightgray}3.5 $\pm$ 2.5   \\ \bottomrule \\
	\end{tabular}}
	
	{\footnotesize \begin{tabular}{ccccccc|ccccc}
			\toprule
			2 & MLP         & GBR         & LR         & DT         & RF         & SGD        & L-DML       & K-DML       & CF          & L-DRL       & F-DRL       \\ \midrule
			\multirow{2}{*}{2k}  & 27.9 $\pm$ 11.4 & 29.1 $\pm$ 12.7 & 27.9 $\pm$ 10.0  & 29.1 $\pm$ 11.0  & 28.0 $\pm$ 9.8   & 35.7 $\pm$ 13.0  & 29.8 $\pm$ 11.1 & 39.7 $\pm$ 10.4 & 27.2 $\pm$ 12.3 & 28.8 $\pm$ 11.7 & 32.4 $\pm$ 11.2 \\
			& \cellcolor{lightgray}14.1 $\pm$ 9.2  & \cellcolor{lightgray}13.4 $\pm$ 9.4  & \cellcolor{lightgray}11.1 $\pm$ 5.6 & \cellcolor{lightgray}12 $\pm$ 8.2   & \cellcolor{lightgray}11.7 $\pm$ 6.9 & \cellcolor{lightgray}14.7 $\pm$ 6.1 & \cellcolor{lightgray}13.4 $\pm$ 13.3 & \cellcolor{lightgray}16.9 $\pm$ 11.5 & \cellcolor{lightgray}12.9 $\pm$ 11.2 & \cellcolor{lightgray}11.8 $\pm$ 8.6  & \cellcolor{lightgray}10.4 $\pm$ 8.2  \\ \midrule
			\multirow{2}{*}{4k}  & 30.2 $\pm$ 10.7 & 29.2 $\pm$ 8.6  & 30.8 $\pm$ 8.9 & 30.2 $\pm$ 8.6 & 29.3 $\pm$ 8   & 32 $\pm$ 8.9   & 26.9 $\pm$ 8.0    & 35.9 $\pm$ 8.9  & 29.6 $\pm$ 8.9  & 27.6 $\pm$ 10.7 & 29.6 $\pm$ 9.6  \\
			& \cellcolor{lightgray}8.7 $\pm$ 7.6   & \cellcolor{lightgray}8.9 $\pm$ 6.3   & \cellcolor{lightgray}8.9 $\pm$ 6.5  & \cellcolor{lightgray}8.5 $\pm$ 6.2  & \cellcolor{lightgray}8.4 $\pm$ 5.5  & \cellcolor{lightgray}9.8 $\pm$ 7.1  & \cellcolor{lightgray}9.6 $\pm$ 7.5   & \cellcolor{lightgray}12.8 $\pm$ 8.8  & \cellcolor{lightgray}9.2 $\pm$ 5.1   & \cellcolor{lightgray}10.9 $\pm$ 9.4  & \cellcolor{lightgray}6.5 $\pm$ 5.1   \\ \midrule
			\multirow{2}{*}{6k}  & 30.9 $\pm$ 6.8  & 28.8 $\pm$ 6.5  & 29.7 $\pm$ 5.4 & 27.7 $\pm$ 5.3 & 30.8 $\pm$ 6.1 & 30.0 $\pm$ 8.9   & 31.3 $\pm$ 5.9  & 36.2 $\pm$ 6.7  & 27.5 $\pm$ 7.2  & 26.2 $\pm$ 6.7  & 27.1 $\pm$ 8.1  \\
			& \cellcolor{lightgray}6.5 $\pm$ 6.9   & \cellcolor{lightgray}5.8 $\pm$ 5.4   & \cellcolor{lightgray}4.4 $\pm$ 4.5  & \cellcolor{lightgray}5.4 $\pm$ 4.7  & \cellcolor{lightgray}6.5 $\pm$ 5.3  & \cellcolor{lightgray}7.6 $\pm$ 6.6  & \cellcolor{lightgray}6.5 $\pm$ 4.8   & \cellcolor{lightgray}10.3 $\pm$ 7.5  & \cellcolor{lightgray}7.7 $\pm$ 7.6   & \cellcolor{lightgray}7.4 $\pm$ 6.1   & \cellcolor{lightgray}6.7 $\pm$ 5.2   \\ \midrule
			\multirow{2}{*}{8k}  & 31.4 $\pm$ 7.7  & 29.2 $\pm$ 6.6  & 31.0 $\pm$ 6.5   & 30.0 $\pm$ 6.4   & 28.8 $\pm$ 6.1 & 32.8 $\pm$ 7.2 & 29.3 $\pm$ 6.6  & 32.4 $\pm$ 5.7  & 28.5 $\pm$ 6.6  & 28.9 $\pm$ 6.3  & 27.9 $\pm$ 5.8  \\
			& \cellcolor{lightgray}8.8 $\pm$ 5.6   & \cellcolor{lightgray}6.4 $\pm$ 4.6   & \cellcolor{lightgray}7.1 $\pm$ 5.9  & \cellcolor{lightgray}6.6 $\pm$ 4.8  & \cellcolor{lightgray}6.5 $\pm$ 4.5  & \cellcolor{lightgray}7.6 $\pm$ 6.0    & \cellcolor{lightgray}5.6 $\pm$ 4.0     & \cellcolor{lightgray}7.4 $\pm$ 4.1   & \cellcolor{lightgray}4.8 $\pm$ 3.6   & \cellcolor{lightgray}6.3 $\pm$ 5.1   & \cellcolor{lightgray}5.2 $\pm$ 3.5   \\ \midrule
			\multirow{2}{*}{10k} & 30.4 $\pm$ 6.2  & 30.9 $\pm$ 5.8  & 28.6 $\pm$ 5.4 & 28.0 $\pm$ 5.4   & 28.8 $\pm$ 4.6 & 28.1 $\pm$ 6.1 & 28.8 $\pm$ 5.9  & 32.3 $\pm$ 4.8  & 30.7 $\pm$ 4.5  & 29.0 $\pm$ 4.6    & 29.2 $\pm$ 5.7  \\
			& \cellcolor{lightgray}6.0 $\pm$ 4.7     & \cellcolor{lightgray}5.4 $\pm$ 3.6   & \cellcolor{lightgray}5.2 $\pm$ 4.2  & \cellcolor{lightgray}4.6 $\pm$ 3.9  & \cellcolor{lightgray}4.9 $\pm$ 3.6  & \cellcolor{lightgray}6.0 $\pm$ 3.9    & \cellcolor{lightgray}6.3 $\pm$ 4.8   & \cellcolor{lightgray}5.7 $\pm$ 5.1   & \cellcolor{lightgray}4.5 $\pm$ 3.0     & \cellcolor{lightgray}5.6 $\pm$ 3.5   & \cellcolor{lightgray}5.9 $\pm$ 4.5   \\ \midrule
			\multirow{2}{*}{20k} & 29.3 $\pm$ 4.5  & 29.2 $\pm$ 3.8  & 29.5 $\pm$ 3.9 & 30.0 $\pm$ 3.4   & 30.4 $\pm$ 2.8 & 29.8 $\pm$ 6.0   & 28.9 $\pm$ 3.9  & 33.4 $\pm$ 3.0    & 29.8 $\pm$ 3.1  & 30.1 $\pm$ 3.2  & 29.4 $\pm$ 4.1  \\
			& \cellcolor{lightgray}4.9 $\pm$ 3.6   & \cellcolor{lightgray}3.5 $\pm$ 2.6   & \cellcolor{lightgray}3.7 $\pm$ 2.9  & \cellcolor{lightgray}3.6 $\pm$ 2.7  & \cellcolor{lightgray}3.4 $\pm$ 2.5  & \cellcolor{lightgray}5.7 $\pm$ 4.9  & \cellcolor{lightgray}4.4 $\pm$ 2.9   & \cellcolor{lightgray}5.3 $\pm$ 3.6   & \cellcolor{lightgray}3.1 $\pm$ 2.2   & \cellcolor{lightgray}3.2 $\pm$ 2.4   & \cellcolor{lightgray}3.8 $\pm$ 3.3   \\ \bottomrule
	\end{tabular}}
	
	\caption{Biases of CDE and CATE estimations on two synthetic data sets in two tables. The rows with / without the background report results with parent / all variables as inputs.  Using parents produces lower biases than using all variables for all methods. The left hand side results with the background are CDE estimations from MODE with various machine learning methods. The right hand results are CATE estimations using the benchmark methods. The biases of MODE and benchmark methods are consistent when considering their confidence intervals.  
	}
	\label{tab:DCE}
\end{table*}

\section{Experiments}

The experiments will demonstrate the following main conclusions in the paper: 1) When the conditions are satisfied, any predictive model is capable of estimating the CDE of any feature on $Y$; 2) a causally interpretable predictive model built on all direct causes (parents of $Y$) will be more robust than a predictive model built on all features variables; 3) a causally interpretable predictive model is potentially useful for personalised decision making.  

\subsection{Capability for causal effect estimation}

Evaluating causal effect estimation is challenging since there are no ground truth causal effects in most data sets. We use two synthetic data sets for evaluations. The data generation details are in the following.  

$\mathcal{G}_1: U \to X_1 \to Y, X \to X_2 \to Y, U \to X_3 \to Y, U \to X_4 \to Y, U  \to X_5$ and $\mathcal{G}_2: U \to X_1 \to Y, X \to X_2 \to Y, U \to X_3 \to Y, U \to X_4 \to Y, U \to X_5, X_1 \to X_2 \to X_3,  X_4 \to X_3$ where $U$ is an unobserved variable. 

The data generation process for data sets from $\mathcal{G}_1$ is described as follows:\begin{gather*}	U \sim Normal(0,1);~~e \sim Normal(0,1);\\	X_1 \sim Binomial(1,\frac{1}{1+exp(-U)});\\	X_2 \sim Binomial(1,\frac{1}{1+exp(-U)});\\	X_3 \sim Binomial(1,\frac{1}{1+exp(-U)});\\	X_4 \sim Binomial(1,\frac{1}{1+exp(-U)});\\	X_5 \sim Binomial(1,\frac{1}{1+exp(-U)});\\	Y \sim 1 + 1.5X_1 + 1.5X_2 + 1.5X_3 + 1.5X_4 + e\end{gather*}

The data generation process for data sets from $\mathcal{G}_2$ is described as follows:\begin{gather*}	U \sim Normal(0,1);~~e \sim Normal(0,1);\\	X_1 \sim Binomial(1,\frac{1}{1+exp(-U)});\\	X_2 \sim Binomial(1,\frac{1}{1+exp(-U\times X_1)});\\	
	X_3 \sim Binomial(1,\frac{1}{1+exp(-U\times X_2\times X_4)});\\	
	X_4 \sim Binomial(1,\frac{1}{1+exp(-U)});\\	
	X_5 \sim Binomial(1,\frac{1}{1+exp(-U)});\\	Y \sim 1 + 1.5X_1 + 1.5X_2 + 1.5X_3 + 1.5X_4 + e\end{gather*}

Ground truth causal effect are calculated in data using \emph{do} calculus rules~\cite{Pearl2009_Book}.   

We instantiate MODE with various well known predictive machine learning models:  Linear Regression (LR)~\cite{gelman2006data},  Decision Tree (DT)~\cite{Quinlan93}, Random Forest (RF)~\cite{breiman2001random}, Stochastic Gradient Descent (SGD)~\cite{bottou2010large}, Multi-Layer Perception (MLP)~\cite{rumelhart1986learning} and Gradient Boosting for regression (GBR)~\cite{friedman2001greedy}. They are implemented in scikit-learn \cite{scikit-learn}.

We compare MODE with some well known CATE estimation methods: Linear Double Machine Learning (L-DML)~\cite{chernozhukov2016double},  Kernel Double Machine Learning (K-DML)~\cite{nie2021quasi}, Causal Forest Double Machine Learning (CF)~\cite{wager2018estimation}, Linear Doubly Robust Learning (L-DRL)~\cite{chernozhukov2016double}, and Forest Doubly Robust Learning (F-DRL)~\cite{athey2019generalized}. They are implemented in EconML\cite{econml}. When the assumptions are satisfied, the results of CDE and CATE are consistent based on our results, and hence the comparison is proper. Since these methods need a pre-specified treatment, we use $X_1$ as the treatment in all comparisons. 


The biases of estimated causal effects are listed in Table~\ref{tab:DCE}. Each estimate is the average in 30 data sets. We have the following conclusions.  

Firstly, when the conditions are satisfied, a predictive model can estimate causal effects as accurately as other tailor designed causal effect estimation methods, most of which were designed in recent years. When considering confidence intervals, the biases of the two groups of methods are consistent.

Secondly, the use of parents of $Y$ only as input variables improves causal effect estimations greatly for all methods. 

Thirdly, all methods are benefited from the increase in data set size.     

The results demonstrate that MODE with a predictive method is capable of estimating causal effects. 

\subsection{Robustness of a ML model using parents}

\begin{table}[t]
	\centering
	\setlength\tabcolsep{1.5pt}
	{\small\begin{tabular}{ccccccc}
			\toprule
			\multirow{2}{*}{Model} &  & \multicolumn{2}{c}{Same environments}      &  & \multicolumn{2}{c}{Changed environments}   \\ \cmidrule{3-4} \cmidrule{6-7} 
			&  & $0 \to 0$           & $1 \to 1$            &  & $0 \to 1$            & $1 \to 0$            \\ \midrule
			\multirow{2}{*}{MLP}   &  & 0.614 $\pm$ 0.024 & 0.614 $\pm$ 0.019 &  & 1.684 $\pm$ 0.123 & 1.551 $\pm$ 0.141 \\
			&  & \cellcolor{lightgray}0.999 $\pm$ 0.038 & \cellcolor{lightgray}0.997 $\pm$ 0.034 &  & \cellcolor{lightgray}1.000 $\pm$ 0.038     & \cellcolor{lightgray}0.997 $\pm$ 0.033 \\ \cmidrule{1-1} \cmidrule{3-4} \cmidrule{6-7} 
			\multirow{2}{*}{GBR}   &  & 0.816 $\pm$ 0.045 & 0.810 $\pm$ 0.042  &  & 3.817 $\pm$ 0.491 & 3.859 $\pm$ 0.591 \\
			&  & \cellcolor{lightgray}1.123 $\pm$ 0.053 & \cellcolor{lightgray}1.131 $\pm$ 0.034 &  & \cellcolor{lightgray}1.122 $\pm$ 0.051 & \cellcolor{lightgray}1.132 $\pm$ 0.033 \\ \cmidrule{1-1} \cmidrule{3-4} \cmidrule{6-7} 
			\multirow{2}{*}{LR}    &  & 0.612 $\pm$ 0.022 & 0.612 $\pm$ 0.027 &  & 1.566 $\pm$ 0.056 & 1.550 $\pm$ 0.078  \\
			&  & \cellcolor{lightgray}0.992 $\pm$ 0.037 & \cellcolor{lightgray}1.008 $\pm$ 0.045 &  & \cellcolor{lightgray}0.992 $\pm$ 0.036 & \cellcolor{lightgray}1.008 $\pm$ 0.045 \\ \cmidrule{1-1} \cmidrule{3-4} \cmidrule{6-7} 
			\multirow{2}{*}{DT}    &  & 0.885 $\pm$ 0.042 & 0.877 $\pm$ 0.057 &  & 4.478 $\pm$ 0.663 & 4.100 $\pm$ 0.443   \\
			&  & \cellcolor{lightgray}1.000 $\pm$ 0.040      & \cellcolor{lightgray}0.995 $\pm$ 0.027 &  & \cellcolor{lightgray}1.000 $\pm$ 0.041     & \cellcolor{lightgray}0.996 $\pm$ 0.028 \\ \cmidrule{1-1} \cmidrule{3-4} \cmidrule{6-7} 
			\multirow{2}{*}{RF}    &  & 0.756 $\pm$ 0.049 & 0.740 $\pm$ 0.037  &  & 3.957 $\pm$ 0.463 & 3.900 $\pm$ 0.614   \\
			&  & \cellcolor{lightgray}1.014 $\pm$ 0.033 & \cellcolor{lightgray}1.016 $\pm$ 0.040  &  & \cellcolor{lightgray}1.013 $\pm$ 0.034 & \cellcolor{lightgray}1.018 $\pm$ 0.040  \\ \cmidrule{1-1} \cmidrule{3-4} \cmidrule{6-7} 
			\multirow{2}{*}{SGD}   &  & 0.700 $\pm$ 0.062   & 0.701 $\pm$ 0.037 &  & 2.417 $\pm$ 0.526 & 2.531 $\pm$ 0.567 \\
			&  & \cellcolor{lightgray}1.003 $\pm$ 0.035 & \cellcolor{lightgray}1.002 $\pm$ 0.035 &  & \cellcolor{lightgray}1.002 $\pm$ 0.034 & \cellcolor{lightgray}1.002 $\pm$ 0.035 \\ \bottomrule
	\end{tabular}}
	\caption{Robustness of predictive models with the changed environments (right) benchmarked by without environment changes (left). The rows with / without the background indicate errors of models with parents / all variables as features. Error differences (right vs. left) of the models using parents are significantly smaller than those of the models using all variables.}
	\label{tab:Binary}
\end{table}

We test the robustness of models using parents as features using synthetic data sets. Data sets are generated using the causal DAG in Figure~\ref{fig_winequality}(a) with the following extensions: Factors contains 3 variables and they share unobserved variable $U_1$. The generation details are in the following.


\begin{gather*}	
	U_1 \sim Normal(0,1);~~U_2 \sim Binomial(0,0.5);\\	
	e_1 \sim Normal(0,1);~~e_2 \sim Normal(0,1);\\	
	X_1 \sim Binomial(1,\frac{1}{1+exp(-U_1)});\\	
	X_2 \sim Binomial(1,\frac{1}{1+exp(-U_1)});\\	
	X_3 \sim Binomial(1,\frac{1}{1+exp(-U_1)});\\	
	Y \sim 1 + 10X_1 + 10X_2 + 10X_3 + e_1;\\	
	P \sim 1 + 0.8*Y + 2*U_2 + e_2
\end{gather*}

 The environments are indicated by $U_2 = 0$ and $U_2 = 1$. ``Without environment changes" means that both training and test data sets are sampled from a data set of one environment, and ``with environment changes" means that training and test data sets are sampled from the data sets of two environments respectively.

Errors of six predictive models are shown in Table~\ref{tab:Binary}. With an unchanged environment, using all variables as features produces models with lower errors than using parents as features since the child node $X_2$ helps predict $Y$. However, when environments are changed, the performance of models using all variables deteriorates significantly, whereas the performance of models using parents remains at the same level. 

The results demonstrate that using parents only will produce robustness models that are not affected by the change of environments (unobserved variables).

We then test the robustness of models using parents as features using three real world data sets for lung cancer studies. Data sets came from the labs of Harvard~\cite{Harvarddata}, Michigan \cite{MichiganData}, and Stanford~\cite{StandfordData}. Data were collected from different patients in three different environments using different devices. Data contain gene expression values of thousands of genes and a class label with cancer and normal. 

The common genes are matched among three data sets and binarised by the median of each gene. 90 genes are selected based on the correlation with the class label using the Harvard data set. Michigan and Stanford data sets are projected to the same set of selected genes and are used as test data sets. 

Here is a summary of the data sets after the processing in the binary format. All data sets have 90 variables. The sizes of Harvard, Michigan, and Stanford data sets are 156 (17), 96 (10) and 46 (5) respectively where the number in the parentheses indicates the number of cancer cases in each data set.  

Models are built on the Harvard data set and are tested on three data sets. The training and test data sets are split based on the $70:30$ ratio. For the same environment, a test data set from Harvard is used. For the changed environments, test data sets (also sampled 30\% from each data set) from Michigan and Stanford are used. The number of repetitions is 30.  
The PC algorithm implemented in the causal discovery toolbox TETRAD~\cite{ramsey2018tetrad} is used to find parents in the Harvard data set. Three parents are found in the Harvard data set with the significant level $\alpha = 0.05$. 

The errors of some well know machine models on the same environment and changed environments are shown in Table~\ref{tab:Real}. 

By comparing the error differences between the changed and unchanged environment of each method, models using parents as features have significantly lower differences than models using all genes as features. This shows that the models using parents are more robust than models using all features.

\begin{table}[h]
	\setlength\tabcolsep{1pt}
	\centering
	{\small\begin{tabular}{cccccc}
			\toprule
			\multirow{2}{*}{Model} &  & Same              &  & \multicolumn{2}{c}{Changed}           \\ \cmidrule{3-3} \cmidrule{5-6} 
			&  & $\text{Harvard} \to \text{Harvard}$ &  & $\text{Harvard} \to \text{Michigan}$ & $\text{Harvard} \to \text{Stanford}$ \\ \midrule
			\multirow{2}{*}{MLP} &  & 0.0262 $\pm$ 0.0191   &  & 0.0631 $\pm$ 0.0435    & 0.0901 $\pm$ 0.0541  \\
			&  & \cellcolor{lightgray}0.0362 $\pm$ 0.0291   &  & \cellcolor{lightgray}0.0447 $\pm$ 0.0384    & \cellcolor{lightgray}0.0603 $\pm$ 0.0291  \\ \cmidrule{1-1} \cmidrule{3-3} \cmidrule{5-6} 
			\multirow{2}{*}{GBR} &  & 0.0184 $\pm$ 0.0165   &  & 0.0376 $\pm$ 0.0289    & 0.1050 $\pm$ 0.0340    \\
			&  & \cellcolor{lightgray}0.0447 $\pm$ 0.0252   &  & \cellcolor{lightgray}0.0284 $\pm$ 0.0196    & \cellcolor{lightgray}0.0716 $\pm$ 0.0368  \\ \cmidrule{1-1} \cmidrule{3-3} \cmidrule{5-6} 
			\multirow{2}{*}{LR}  &  & 0.0262 $\pm$ 0.0221   &  & 0.0291 $\pm$ 0.0220     & 0.1043 $\pm$ 0.0281  \\
			&  & \cellcolor{lightgray}0.0170 $\pm$ 0.0162    &  & \cellcolor{lightgray}0.0241 $\pm$ 0.0145    & \cellcolor{lightgray}0.0199 $\pm$ 0.0167  \\ \cmidrule{1-1} \cmidrule{3-3} \cmidrule{5-6} 
			\multirow{2}{*}{DT}  &  & 0.0362 $\pm$ 0.0275   &  & 0.1085 $\pm$ 0.0757    & 0.0858 $\pm$ 0.0456  \\
			&  & \cellcolor{lightgray}0.0461 $\pm$ 0.0231   &  & \cellcolor{lightgray}0.0362 $\pm$ 0.0327    & \cellcolor{lightgray}0.0525 $\pm$ 0.0374  \\ \cmidrule{1-1} \cmidrule{3-3} \cmidrule{5-6} 
			\multirow{2}{*}{RF}  &  & 0.0220 $\pm$ 0.0153    &  & 0.0376 $\pm$ 0.0260     & 0.1199 $\pm$ 0.0540   \\
			&  & \cellcolor{lightgray}0.0291 $\pm$ 0.0240    &  & \cellcolor{lightgray}0.0227 $\pm$ 0.0223    & \cellcolor{lightgray}0.0652 $\pm$ 0.0353  \\ \cmidrule{1-1} \cmidrule{3-3} \cmidrule{5-6} 
			\multirow{2}{*}{SGD} &  & 0.0426 $\pm$ 0.0262   &  & 0.0844 $\pm$ 0.0643    & 0.1234 $\pm$ 0.0866  \\
			&  & \cellcolor{lightgray}0.0369 $\pm$ 0.0301   &  & \cellcolor{lightgray}0.0418 $\pm$ 0.0328    & \cellcolor{lightgray}0.0496 $\pm$ 0.0359  \\ \bottomrule
	\end{tabular}}
	\caption{Errors of predictive models using all features (no background) / parents only (with the background) in three real world data sets obtained from different labs. The error difference between environment change and no environment change indicates the robustness of the model. The models using parents have significantly smaller differences than models using all features. }
	\label{tab:Real}
\end{table}

\subsection{Illustration for personalised decision making}

\begin{table}[t]
	\centering
	\setlength\tabcolsep{1.5pt}
	{\small\begin{tabular}{ccc}
			\toprule
			\multicolumn{3}{c}{A person from German Credit Data Set}                            \\ \midrule
			Original Outcome              &  & Bad credit at 95\% probability                     \\ \midrule
			\multirow{2}{*}{Feature}      &  & Savings.account.1 = 0                              \\
			&  & (Saving account no more than 100 DM)                  \\ 
			CDE                           &  & 0.93                                               \\ 
			\multirow{2}{*}{Intervention} &  & $do$(Savings.account.1 = 1)                          \\
			&  & (Make saving account \textgreater 100 DM) \\ 
			New Outcome                   &  & Good credit at 98\% probability                    \\ \midrule
			Feature                       &  & Housing.2 = 0 (Not own a house)                    \\ 
			CDE                           &  & 0.88                                               \\ 
			Intervention                  &  & $do$(Housing.2 = 1) (Own a house)                    \\ 
			New Outcome                   &  & Good credit at 93\% probability                    \\ \midrule
			Feature                       &  & Credit.history.5 = 0 (No other credits existing)   \\ 
			CDE                           &  & 0.85                                               \\ 
			Intervention                  &  & $do$(Credit.history.5 = 1) (Have other credits)      \\ 
			New Outcome                   &  & Good credit at 90\% probability                    \\ \bottomrule
	\end{tabular}}
	\caption{An illustration for personalised decision making using MODE. The second row shows the outcome predicted for a person. Three features (with the original input values) of the largest CDEs are shown. The person can change the credit rating by an intervention $do()$ on any feature. The outcome after an intervention is shown as New Outcome.}
	\label{tab:ExampleDecisionMaking}
\end{table}

For the purpose of illustrating how MODE can be used for personalised decision making, we assume the assumptions are satisfied in the German credit data set~\cite{Bache+Lichman:2013}. 
In the German credit data set~\cite{Bache+Lichman:2013}, each value in a categorical attribute is converted to a binary variable where 1 indicates the presence of the value. A numerical attribute is binarised by the median. In total, the data set includes 61 binary variables including the class variable. 
Seven parents are identified in the data set using PC-Select~\cite{PC-select-2010} with the significant level $\alpha=0.05$. A random forest of 500 trees is built on the data set. 


A case to illustrate MODE for personalised decision making is shown in Table~\ref{tab:ExampleDecisionMaking}. The person is predicted as ``bad credit", and the top three features with the highest CDEs give the person a choice to choose one to work on for improving his/her credit rating in future.

\section{Conclusion and discussion}

We have presented the assumptions and theories to support estimating the CDE of any feature on the outcome. Our assumptions relax the assumption of knowing the complete causal graph underlying the data set for estimating the causal effect of any feature on the outcome. We assume that the data set includes all direct causes of the outcome and no variables affected by the outcome. We allow hidden variables in the system as long as there are no hidden direct causes of $Y$.  Our assumptions support practical methods for estimating CDE of any feature on the outcome. We further show that when the conditions are satisfied, any type of predictive models can be used to estimate the CDE of any feature on the outcome.

The results in the paper make the predictive model satisfying the conditions identified by us causally interpretable since each feature's contribution to the outcome change in a circumstance (i.e., CDE) is estimable. Furthermore, we show such a causally interpretable model is robust in terms of making predictions in a new environment that is different from the one from which the model is trained.

The proposed approach relies on assumptions for its correctness, and the assumptions need domain knowledge to verify. Note that all causal effect estimation methods need some assumptions and the set of assumptions in this paper are more relaxed than assuming knowing the causal graph underlying the data.

The implication of the work in this paper is twofold. Firstly, this set of assumptions can be easily understood by users to enable them to know whether the causal effect can be estimated in data using a predictive model. Secondly, in many applications, users know most direct causes of the outcome, and a certain level of biases in estimated causal effects is acceptable. For example, in product recommendation, knowing whether an action has a positive effect on the sale is important already even without knowing the precise magnitude of the causal effect. In these applications, the method presented will assist users in both outcome prediction and finding suitable treatments (actions) for each individual to change the outcome. 

In summary, causally interpreting a predictive model needs domain knowledge. The results in this paper provide an easy means for integrating the top level domain knowledge (the data set containing all direct causes of the outcome and no variables affected by the outcome) instead of detailed domain knowledge (the complete causal  graph). The results enable human guided machine learning for building causally interpretable and robust models.

\bibliographystyle{abbrv}

\end{document}